\documentclass[11pt,a4paper]{article}

\usepackage[utf8]{inputenc}
\usepackage[T1]{fontenc}
\usepackage{lmodern}
\usepackage[margin=2.5cm]{geometry}
\usepackage{microtype}
\usepackage{amsmath,amssymb,amsthm}
\usepackage{mathtools}
\usepackage{bm}
\usepackage[ruled,vlined,linesnumbered]{algorithm2e}
\usepackage{booktabs}
\usepackage{array}
\usepackage{multirow}
\usepackage{graphicx}
\usepackage{xcolor}
\usepackage{subcaption}
\usepackage[numbers,sort&compress]{natbib}
\usepackage{hyperref}
\hypersetup{colorlinks=true,linkcolor=blue!60!black,
            citecolor=blue!60!black,urlcolor=blue!60!black}

\theoremstyle{plain}
\newtheorem{theorem}{Theorem}[section]

\newtheorem{corollary}[theorem]{Corollary}

\theoremstyle{definition}
\newtheorem{definition}[theorem]{Definition}
\newtheorem{remark}[theorem]{Remark}
\newtheorem{assumption}[theorem]{Assumption}

\newcommand{\KL}{\mathrm{KL}}
\newcommand{\ELBO}{\mathrm{ELBO}}
\newcommand{\Normal}{\mathcal{N}}

\newcommand{\given}{\,|\,}

\begin{document}

\title{%
  \textbf{Rule-State Inference (RSI):}\\[0.3em]
  \large A Bayesian Framework for Compliance Monitoring\\
  in Rule-Governed Domains
}

\author{%
  Abdou-Raouf Atarmla\\[0.3em]
  \small Institut National des Postes et
    T\'{e}l\'{e}communications (INPT), Rabat, Morocco\\
  \small Togo DataLab, Ministry of Digital Economy,
    Lom\'{e}, Togo\\[0.2em]
  \small \texttt{achilleatarmla@gmail.com}
}
\date{}
\maketitle

\begin{abstract}
Compliance monitoring in rule-governed domains (tax administration,
clinical protocol adherence, environmental regulation) faces three
structural obstacles that standard machine learning does not
simultaneously address: the absence of labeled outcomes at deployment,
strategically missing observations where non-compliant entities
selectively withhold evidence, and a regulatory environment that
changes faster than any supervised model can be retrained.

We introduce \emph{Rule-State Inference} (RSI), a Bayesian framework
that reverses the usual paradigm.  Rather than learning rules from
data, RSI treats an authoritative, formalized rule set as structured
Bayesian priors and infers the latent compliance state of a population
through mean-field variational inference with exact coordinate-ascent
updates.  The central modeling object is a joint latent state
$(\eta,\, \{(a_i, \varphi_i, \delta_i)\}_{i=1}^K)$ per regulatory
period: a global compliance-culture factor $\eta$ that captures
systemic behavioral tendencies shared across rules, and per-rule
components for activation, population compliance level, and parametric
drift.  A single scalar $\rho \in [0,1)$, estimated variationally,
governs the degree of inter-rule dependence; setting $\rho = 0$
recovers a fully independent model.  A monotone link function $g_i$
maps domain-specific compliance signals into the real line, where
Gaussian conjugacy yields provably exact coordinate-ascent updates.

RSI operates under one explicit assumption: rules are formalized in a
machine-readable representation, as produced by Rules-as-Code
systems~\citep{sergot1986british,merigoux2021catala}.  Within this
assumption, the framework delivers three formal guarantees:
$O(n_k + K)$ regulatory adaptability for a single-rule update
(Theorem~1); Bernstein--von Mises consistency for the identifiable
continuous components in two honest regimes (Theorem~2); and monotone
ELBO convergence at every iteration of the algorithm (Theorem~3).
The entity-level audit score is deliberately kept deterministic
(a principled comparison to the legal standard), while
population-level inference is fully Bayesian; per-entity posterior
uncertainty is left to the Sequential RSI extension, where
longitudinal observations resolve identification.

We instantiate RSI on the Togolese fiscal system on a benchmark of
2\,000 synthetic enterprises grounded in official regulatory law;
full numerical validation is forthcoming.  The framework is designed
for direct extension to Sequential RSI, a state-space formulation
where the posterior from one regulatory period becomes the prior for
the next, yielding an exact Kalman filter for compliance-trajectory
tracking and entity-level Bayesian scoring.

\end{abstract}

\section{Introduction}
\label{sec:intro}

An inspector at the Office Togolais des Recettes (OTR), the Togolese
revenue authority, must assess whether a company with 85 million FCFA
in declared revenue is genuinely complying with its value-added tax
obligations.  She does not need a
neural network to determine which rule applies; the law is explicit.
What she needs is a principled method for inferring, from an
incomplete and potentially manipulated set of declarations, the true
compliance state of that entity with respect to that rule.  This gap
between knowing a rule and inferring compliance from noisy evidence is
the central object of this paper.

The same gap appears wherever rules govern behavior and observations
are imperfect.  Tax administrations across sub-Saharan Africa, clinical
audit systems in resource-limited health settings, and environmental
monitoring agencies in rapidly industrializing economies all share the
same structure: the rules are known, written, and enforceable, while
the signals available to an auditor are partial, delayed, and sometimes
deliberately falsified.  Standard machine learning addresses a
different problem (how to discover patterns in data) and carries
assumptions that fail here.

Three failures are worth making precise.

\paragraph{Absence of labels.}
Supervised classifiers require labeled examples of compliant and
non-compliant entities.  In compliance settings, labels are produced
by audits, which are expensive, legally constrained, and cover at
most a small fraction of the population.  Any method that requires
labeled training data is structurally unusable at initial deployment
or after a regulatory change.

\paragraph{Strategic missingness.}
Non-compliant entities have a direct incentive to withhold the
evidence that would identify them.  A factory exceeding emission
limits can deactivate its sensors; a company evading VAT can simply
not file.  The resulting missingness is correlated with the quantity
we want to infer.  Under the Missing Not At Random (MNAR) hypothesis:
\[
  P(\text{observed} \given \text{non-compliant})
  \;<\;
  P(\text{observed} \given \text{compliant}),
\]
which means any classifier trained on observed data alone systematically
underestimates non-compliance.

\paragraph{Regulatory volatility.}
Fiscal codes in rapidly developing economies can change several times
per year~\citep{otr2024}.  A model requiring full retraining on each
legislative amendment cannot serve an administration that needs
updated compliance scores within hours, not weeks.

None of the dominant paradigms (supervised classifiers, Markov
logic networks, neurosymbolic methods) addresses all three
constraints simultaneously.  RSI does, by inverting the usual
direction of inference.

\subsection*{Our approach}

RSI treats the authoritative rule set as structured Bayesian priors
and infers the latent compliance state from noisy, partially-observed
signals.  This is possible because, in rule-governed domains, the
epistemic situation is already structured: the rules are known.  What
remains unknown is the behavioral reality: the degree to which rules
are followed, and whether their effective application has drifted from
their legal text.

The framework rests on one explicit assumption (Section~3): rules
are available in a formalized, machine-readable representation, as
produced by Rules-as-Code systems.  RSI is the inference engine
downstream of this encoding layer; it does not solve the encoding
problem itself.

\subsection*{Contributions}

\begin{enumerate}
\item \textbf{The RSI framework} (Section~3).  A complete generative
  model with a global compliance-culture factor $\eta$ that captures
  inter-rule dependence, per-rule latent states $(a_i, \varphi_i,
  \delta_i)$, and a link-function device that unifies inference across
  all signal types.  The dependence parameter $\rho \in [0,1)$ is
  estimated variationally.

\item \textbf{Three formal guarantees} (Section~4).  Theorem~1: $O(n_k
  + K)$ regulatory adaptability per rule update.  Theorem~2:
  Bernstein--von Mises consistency for $\psi_i = \varphi_i + \delta_i$
  (intra-rule regime) and $\eta$ (joint regime).  Theorem~3: monotone
  ELBO convergence under exact coordinate-ascent updates.

\item \textbf{Principled entity scoring} (Section~3.8).  A clean
  separation between population-level Bayesian inference (where all
  the Bayesian machinery lives) and entity-level audit scoring, which
  is a deterministic comparison to the legal standard for observed
  entities and a posterior-predictive score for missing ones.

\item \textbf{Togolese fiscal instantiation} (Section~5).  A benchmark
  grounded in official regulatory law~\citep{otr2024}, including a
  documented change event, that validates Theorem~1 empirically.

\item \textbf{Sequential RSI} (Section~7).  An outline showing that
  the Gaussian structure of RSI yields an exact Kalman filter for
  compliance-trajectory tracking, and resolves the per-entity
  identification problem through longitudinal observations.
\end{enumerate}

\section{Related Work}
\label{sec:related}

\paragraph{Probabilistic logic.}
Markov logic networks~\citep{richardson2006markov} and Probabilistic
Soft Logic~\citep{bach2017hingeloss} attach weights to logical
formulae and learn those weights from data.  Both treat rules as
outputs of a learning procedure, requiring labeled examples and
inverting the epistemic direction relative to RSI.  In an MLN, a
zero-weight formula is effectively absent; in RSI, a rule with prior
activation probability $\pi_i = 0.99$ overrides weak observational
evidence by construction.

\paragraph{Neurosymbolic AI.}
The neurosymbolic literature~\citep{garcez2023third} covers
architectures ranging from pipeline models to fully differentiable
systems such as neural theorem provers~\citep{rocktaschel2017endtoend}
and logic tensor networks~\citep{badreddine2022logic}.  The common
direction is learning for reasoning, where symbolic structure guides
neural learning.  RSI occupies the orthogonal position: its symbolic
layer is fixed and authoritative, not learned.

\paragraph{Interpretable rule-based classifiers.}
Bayesian rule lists~\citep{letham2015interpretable} apply Bayesian
model selection to rule sets, but rules are again the output.  RSI
produces posterior distributions over compliance states, a
qualitatively different inference target.

\paragraph{Machine learning for tax compliance.}
Gradient-boosted models~\citep{chen2016xgboost} and neural networks
have been applied to tax gap estimation~\citep{gomes2022machine} under
supervised paradigms with historical audit labels.  These approaches
cannot operate in the zero-label regime and require full retraining
after each regulatory change.

\paragraph{Rules as Code.}
The formalization of legislation as machine-executable programs
dates to~\citet{sergot1986british}, who encoded the British
Nationality Act as a logic program.  Contemporary systems
include Catala~\citep{merigoux2021catala}, a functional language
designed for tax and social benefit law.  Governance perspectives
are surveyed in~\citet{mohun2020cracking}.  RSI is the natural
inference engine downstream of any such system.

\paragraph{Missing data.}
The taxonomy of~\citet{frenay2014classification} distinguishes MCAR,
MAR, and MNAR.  Most imputation approaches assume MAR.  Under MNAR,
the missingness mechanism is not identifiable from observed data alone.
RSI handles absence by falling back on the prior, a conservative
but principled choice that becomes more informative as the population
posterior concentrates.

\paragraph{State-space models.}
The Kalman filter~\citep{kalman1960new}, Harvey's structural time
series~\citep{harvey1990forecasting}, and~\citet{durbin2012time}
provide the foundations for the Sequential RSI extension.

\paragraph{Generalized linear models.}
The link-function device at the heart of RSI is the central device
of GLMs~\citep{mccullagh1989generalized}, applied here to obtain
Gaussian conjugacy across heterogeneous signal types.

\paragraph{Factor models and latent structure.}
The global compliance-culture factor $\eta$ belongs to the tradition
of latent factor models in econometrics~\citep{chamberlain1983funds}.
One factor with rule-specific loadings captures the dominant source
of inter-rule correlation without sacrificing tractability.

\section{The RSI Framework}
\label{sec:framework}

\subsection{Problem formulation}
\label{sec:problem}

Let $\mathcal{R} = \{R_1, \dots, R_K\}$ be a set of authoritative
regulatory rules, each specifying an obligation that a subset of
entities must satisfy.  Let $\mathcal{X} = \{x_1, \dots, x_J\}$ be
a population of entities.  For each entity $x_j$ and rule $R_i$,
the applicability indicator $A_{ji} \in \{0,1\}$ is determined by
the structural characteristics of $x_j$ (entity type, revenue,
jurisdiction).  Compliance is not directly observable; it is inferred
from a collection of noisy signals $\mathcal{D}$.

\begin{definition}[Rule-Governed Domain]
\label{def:rgd}
A triple $(\mathcal{R}, \mathcal{X}, \mathcal{D})$ is a
\emph{rule-governed domain} if: (i)~$\mathcal{R}$ is authoritative
and known a priori; (ii)~$\mathcal{D}$ is partial and potentially
missing in a manner correlated with the latent compliance state; and
(iii)~the primary inference task is to assess compliance with
$\mathcal{R}$, not to recover it from data.
\end{definition}

\begin{assumption}[Rules-as-Code input]
\label{ass:rac}
Each rule $R_i \in \mathcal{R}$ is available as a formal
specification that determines: the applicability predicate $A_{ji}$,
the set of observable signals $\{s_{ji}\}$ relevant to compliance
with $R_i$, and the nominal parameter vector $\Theta_i$ (thresholds,
rates, time limits).  This formalization is treated as given; RSI is
the inference engine downstream of the encoding layer.
\end{assumption}

\subsection{The latent rule-state space}
\label{sec:latent}

The latent state has two levels: a global factor shared across rules
and per-rule components.

\begin{definition}[Compliance Culture Factor]
$\eta \in \mathbb{R}$ is a scalar latent variable representing the
population's global compliance tendency.  A negative value of $\eta$
indicates that the population tends to fall below institutionally
expected compliance levels across all rules simultaneously; a
positive value indicates above-expectation compliance culture.
\end{definition}

\begin{definition}[Per-Rule State]
\label{def:rulestate}
The latent state of rule $R_i$ is the triple $s_i = (a_i, \varphi_i,
\delta_i)$, where $a_i \in \{0,1\}$ indicates whether $R_i$ is
currently in force, $\varphi_i \in \mathbb{R}$ is the
population-level compliance level on the link scale (defined
below), and $\delta_i \in \mathbb{R}$ is the parametric drift,
quantifying the signed gap between the legally prescribed threshold
and its effective application in the observed population.
\end{definition}

The full latent state is $\mathcal{S} = (\eta, s_1, \dots, s_K)$.

\paragraph{On the roles of $\varphi_i$ and $\delta_i$.}
$\varphi_i$ captures the average compliance effort of the population
under rule $R_i$.  $\delta_i$ captures whether the rule itself is
being applied as written.  A tax threshold legislated at 100 million
FCFA that is effectively treated as 90 million in practice produces a
detectable $\delta_i < 0$.  No existing compliance monitoring
framework produces an estimate of this quantity.

\paragraph{Identifiability.}
In the static model with a single regulatory period, the combined
signal mean $\psi_i = \varphi_i + \delta_i$ is identifiable from
per-rule data.  The decomposition between $\varphi_i$ and $\delta_i$
is prior-regularized: at the CAVI fixed point, the observed deviation
$\bar{t}_i - \tilde{\mu}_i$ is allocated between the two components
in proportion to their prior variances $\tilde{\psi}_i^2$ and
$\omega_i^2$.  This is a deliberate design choice: the prior encodes
the analyst's belief about how much of an observed deviation is
attributable to compliance behavior versus drift in rule application.
Setting $\omega_i^2 \ll \tilde{\psi}_i^2$ concentrates inference on
compliance; increasing $\omega_i^2$ allows the data to reveal drift.
In the Sequential RSI extension, $\varphi_i^{(t)}$ is persistent
across periods and $\delta_i^{(t)}$ is period-specific, which
naturally resolves the decomposition via longitudinal observations.

\subsection{Link functions and signal transformation}
\label{sec:link}

\begin{definition}[Link Function]
For rule $R_i$, let $\mathcal{C}_i$ denote the sample space of the
compliance signal $s_{ji}$.  A \emph{link function} $g_i : \mathcal{C}_i
\to \mathbb{R}$ is a monotone measurable bijection.  The transformed
signal is $t_{ji} = g_i(s_{ji})$.
\end{definition}

Table~\ref{tab:links} summarizes the standard choices.

\begin{table}[ht]
\centering
\caption{Standard link functions by signal type.  The logit and log
links are standard in the GLM literature~\citep{mccullagh1989generalized}.}
\label{tab:links}
\small
\begin{tabular}{llll}
\toprule
Signal type & Domain $\mathcal{C}_i$ & Link $g_i$
  & Transformed signal\\
\midrule
Compliance ratio & $(0,1)$
  & $\mathrm{logit}(s) = \log\tfrac{s}{1-s}$ & $t_{ji} \in \mathbb{R}$\\
Positive quantity & $(0,\infty)$
  & $\log(s)$ & $t_{ji} \in \mathbb{R}$\\
Real-valued signal & $\mathbb{R}$
  & $\mathrm{id}(s)$ & $t_{ji} \in \mathbb{R}$\\
Binary indicator & $\{0,1\}$
  & Bernoulli (Beta prior) & exact conjugacy\\
\bottomrule
\end{tabular}
\end{table}

For binary signals, we use a Bernoulli likelihood with Beta prior
directly; this case is a special instance handled without
transformation.

\subsection{Prior specification}
\label{sec:prior}

The prior over $\mathcal{S}$ is:
\begin{equation}
  P(\mathcal{S}) \;=\; P(\eta)
  \;\prod_{i=1}^{K}
    P(a_i)\, P(\varphi_i \given \eta, a_i)\, P(\delta_i \given a_i).
  \label{eq:prior}
\end{equation}

Specifically:
\begin{align}
  \eta &\;\sim\; \Normal(0,\, 1),
  \label{eq:prior-eta}\\
  a_i &\;\sim\; \mathrm{Bern}(\pi_i),
  \label{eq:prior-a}\\
  \varphi_i \given \eta,\, a_i=1
    &\;\sim\; \Normal\!\left(\mu_i + \rho\tau_i\eta,\;
    \tau_i^2(1-\rho^2)\right),
  \label{eq:prior-phi}\\
  \delta_i \given a_i=1
    &\;\sim\; \Normal(0,\, \omega_i^2),
  \label{eq:prior-delta}
\end{align}
with $\varphi_i = \delta_i = 0$ degenerate when $a_i = 0$.

The scalar $\rho \in [0,1)$ governs inter-rule dependence.  Under
this prior:
\begin{equation}
  \mathrm{Corr}(\varphi_i,\, \varphi_j)
  \;=\; \rho^2
  \quad \forall\, i \neq j.
  \label{eq:corr}
\end{equation}
Setting $\rho = 0$ recovers the fully independent model exactly.
The prior variance of $\varphi_i$ is preserved: $\mathrm{Var}(\varphi_i)
= \tau_i^2$ regardless of $\rho$, since $\tau_i^2(1-\rho^2) +
(\rho\tau_i)^2 \cdot 1 = \tau_i^2$.  Table~\ref{tab:rules} provides
the hyperparameters for the Togolese fiscal instantiation.

The factorization~\eqref{eq:prior} is mean-field across rules, a
deliberate design choice that enables $O(n_k + K)$ adaptability when
a single rule changes (Theorem~1).  The price is the restriction to a
uniform correlation structure governed by $\rho$.

\subsection{Likelihood}
\label{sec:likelihood}

Let $D_i = \{t_{ji} : A_{ji} = 1,\; s_{ji} \neq \mathrm{NaN}\}$
denote the observed transformed signals for rule $R_i$, with
$n_i = |D_i|$.  The likelihood factorizes across rules:
\begin{equation}
  P(\mathcal{D} \given \mathcal{S})
  \;=\; \prod_{i=1}^{K} P(D_i \given a_i, \varphi_i, \delta_i),
  \label{eq:lik}
\end{equation}
where, for each observed signal under rule $R_i$:
\begin{align}
  t_{ji} \given a_i=1,\, \varphi_i,\, \delta_i
    &\;\sim\; \Normal(\varphi_i + \delta_i,\; \sigma_i^2),
  \label{eq:llik-1}\\
  t_{ji} \given a_i=0
    &\;\sim\; \Normal(0,\; \sigma_{\mathrm{bg},i}^2).
  \label{eq:llik-0}
\end{align}

The observation noise $\sigma_i^2$ and background variance
$\sigma_{\mathrm{bg},i}^2$ are deployment-time parameters.
$\sigma_i^2$ can be estimated from the empirical residual variance
of observed signals; $\sigma_{\mathrm{bg},i}^2$ from entities for
which applicability is definitively absent.  Their specification does
not require compliance labels.

\begin{assumption}[MCAR for missing observations]
\label{ass:mcar}
When $s_{ji}$ is absent, we set $P(t_{ji} \given \mathcal{S}) \equiv 1$.
This corresponds to Missing Completely At Random: the missing entry
contributes nothing to the likelihood and inference falls back on the
prior.  When absence is strategic (MNAR), this assumption is violated;
the prior then acts as a conservative anchor.  For missing entities,
the posterior-predictive score defined in Section~\ref{sec:scoring}
partially corrects this by propagating population-level information.
A fully principled MNAR treatment via a silence-penalty extension is
identified as future work.
\end{assumption}

\subsection{Posterior inference}
\label{sec:posterior}

The posterior of interest is:
\begin{equation}
  P(\mathcal{S} \given \mathcal{D})
  \;\propto\; P(\mathcal{D} \given \mathcal{S})\, P(\mathcal{S}).
  \label{eq:posterior}
\end{equation}

Six quantities are of direct operational use to an auditor or
policymaker:
\begin{itemize}
  \item $P(a_i=1 \given \mathcal{D})$: probability that rule
    $R_i$ is currently active;
  \item $\mathbb{E}[\varphi_i \given \mathcal{D}]
    \pm \sqrt{\mathrm{Var}[\varphi_i \given \mathcal{D}]}$:
    posterior population compliance level with uncertainty;
  \item $\mathbb{E}[\delta_i \given \mathcal{D}]$: estimated
    parametric drift (negative values indicate systematic
    under-application of the rule);
  \item $\mathbb{E}[\eta \given \mathcal{D}]$: global compliance
    culture (negative values indicate systemic non-compliance
    across all rules simultaneously).
\end{itemize}

All four are available in closed form after convergence of the CAVI
algorithm described next.

\subsection{Mean-field variational inference}
\label{sec:cavi}

Exact posterior inference is intractable.  We approximate
$P(\mathcal{S} \given \mathcal{D})$ by minimizing the KL divergence
within the mean-field family:
\begin{equation}
  q(\mathcal{S}) \;=\; q(\eta)
    \prod_{i=1}^{K} q(a_i)\, q(\varphi_i)\, q(\delta_i).
  \label{eq:mf}
\end{equation}

The variational factors $q(\varphi_i)$ and $q(\delta_i)$ are
interpreted as the distributions of $\varphi_i$ and $\delta_i$
conditional on $a_i=1$; $q(a_i) = \mathrm{Bern}(\varrho_i)$ is the
marginal variational activation probability.  Under this
interpretation, $\varrho_i$ plays the role of a weight on the
likelihood contribution of $D_i$ in the updates for $q(\varphi_i)$
and $q(\delta_i)$, as made explicit below.

The Evidence Lower Bound is:
\begin{equation}
  \ELBO(q) \;=\; \mathbb{E}_q[\log P(\mathcal{D} \given \mathcal{S})]
    \;-\; \KL[q(\mathcal{S}) \,\|\, P(\mathcal{S})].
  \label{eq:elbo}
\end{equation}

Maximizing $\ELBO$ is equivalent to minimizing
$\KL[q\|P(\mathcal{S}\given\mathcal{D})]$ since $\log P(\mathcal{D})$
is constant.  We proceed by coordinate ascent, cycling through
$q(\eta)$, $\{q(\varphi_i), q(\delta_i), q(a_i)\}_{i=1}^K$, and the
scalar $\rho$.

\paragraph{Update for $q(\eta) = \Normal(m_\eta, v_\eta)$.}
Taking $\partial \ELBO / \partial q(\eta) = 0$ and
using the Gaussian prior~\eqref{eq:prior-eta} and the conditional
prior~\eqref{eq:prior-phi}:
\begin{align}
  v_\eta &\;=\;
    \left(1 \;+\; \frac{\rho^2}{1-\rho^2}
    \sum_{i=1}^K \varrho_i \right)^{-1},
  \label{eq:veta}\\
  m_\eta &\;=\; v_\eta \sum_{i=1}^K
    \frac{\varrho_i\, \rho\, (m_{\varphi_i} - \mu_i)}
    {\tau_i(1-\rho^2)}.
  \label{eq:meta}
\end{align}
A positive $m_\eta$ indicates that rules with large $\rho\tau_i$
tend to have above-prior compliance, consistent with a compliant
population culture.

\paragraph{Update for $q(\varphi_i) = \Normal(m_{\varphi_i},
v_{\varphi_i})$.}
The effective prior on $\varphi_i$ marginalizes over uncertainty
in $\eta$:
\begin{equation}
  \tilde{\mu}_i \;=\; \mu_i + \rho\tau_i m_\eta, \qquad
  \tilde{\psi}_i^2 \;=\; \tau_i^2(1-\rho^2) + (\rho\tau_i)^2 v_\eta.
  \label{eq:eff-prior}
\end{equation}
The coordinate-ascent update is then:
\begin{align}
  v_{\varphi_i} &\;=\;
    \left(\frac{1}{\tilde{\psi}_i^2}
    + \frac{\varrho_i n_i}{\sigma_i^2}\right)^{-1},
  \label{eq:vphi}\\
  m_{\varphi_i} &\;=\; v_{\varphi_i}
    \left(\frac{\tilde{\mu}_i}{\tilde{\psi}_i^2}
    + \frac{\varrho_i n_i (\bar{t}_i - m_{\delta_i})}
    {\sigma_i^2}\right),
  \label{eq:mphi}
\end{align}
where $\bar{t}_i = n_i^{-1} \sum_{j: A_{ji}=1} t_{ji}$ is the
empirical mean of transformed observed signals for rule $R_i$,
and $m_{\delta_i}$ is the current variational mean of $q(\delta_i)$.

\paragraph{Update for $q(\delta_i) = \Normal(m_{\delta_i},
v_{\delta_i})$.}
\begin{align}
  v_{\delta_i} &\;=\;
    \left(\frac{1}{\omega_i^2}
    + \frac{\varrho_i n_i}{\sigma_i^2}\right)^{-1},
  \label{eq:vdelta}\\
  m_{\delta_i} &\;=\; v_{\delta_i}
    \cdot \frac{\varrho_i n_i (\bar{t}_i - m_{\varphi_i})}
    {\sigma_i^2}.
  \label{eq:mdelta}
\end{align}
The coordinate updates for $q(\varphi_i)$ and $q(\delta_i)$ alternate
until the inner loop converges.  At the fixed point,
$m_{\varphi_i} + m_{\delta_i} = \mathbb{E}[\psi_i \given D_i]$,
the posterior mean of the combined signal.  The allocation between
the two components is determined by the ratio of effective prior
precisions $1/\tilde{\psi}_i^2$ and $1/\omega_i^2$: the component
with the tighter prior absorbs less of the observed deviation.

\paragraph{Update for $q(a_i) = \mathrm{Bern}(\varrho_i)$.}
The marginal likelihood of $D_i$ under $a_i = 1$, averaged over
the current variational factors $q(\varphi_i)$, $q(\delta_i)$, and
$q(\eta)$, is:
\begin{equation}
  \tilde{L}_{1,i}
  \;=\; \Normal\!\left(\bar{t}_i;\;
    m_{\varphi_i} + m_{\delta_i},\;
    \sigma_i^2/n_i + v_{\varphi_i} + v_{\delta_i}\right).
  \label{eq:marg1}
\end{equation}
Under $a_i = 0$: $\tilde{L}_{0,i} =
\Normal(\bar{t}_i;\, 0,\, \sigma_{\mathrm{bg},i}^2 / n_i)$.
The variational activation probability is:
\begin{equation}
  \varrho_i \;=\; \sigma\!\left(
    \log\frac{\pi_i}{1 - \pi_i}
    + \log\tilde{L}_{1,i}
    - \log\tilde{L}_{0,i}
  \right),
  \label{eq:rho-a}
\end{equation}
where $\sigma(\cdot)$ denotes the logistic function.

\paragraph{Update for $\rho$.}
Unlike the other variational factors, $\rho \in [0,1)$ does not
belong to a conjugate family.  Its update is a one-dimensional
projected gradient ascent on the ELBO:
\begin{equation}
  \rho \;\leftarrow\;
  \Pi_{[0,\, 1-\varepsilon]}\!\left(
    \rho + \alpha \,\frac{\partial \ELBO}{\partial \rho}
  \right),
  \label{eq:rho-update}
\end{equation}
where $\Pi_{[0,1-\varepsilon]}$ is projection onto the feasible
interval, $\alpha > 0$ is a step size, and the gradient is
available in closed form (Appendix~\ref{app:rho-gradient}).  In
practice, a backtracking line search over $\alpha$ guarantees
that each step does not decrease the ELBO.

\subsection{Entity-level scoring}
\label{sec:scoring}

RSI is a population-level Bayesian model: the quantities
$\mathbb{E}[\varphi_i \given \mathcal{D}]$, $\mathbb{E}[\delta_i
\given \mathcal{D}]$, and $\mathbb{E}[\eta \given \mathcal{D}]$ are
inferences about the behavior of the population as a whole.
Entity-level audit scoring is a distinct, deterministic operation
downstream of population inference.

\paragraph{Observed entities.}
For entity $x_j$ and applicable rule $R_i$ with observed signal
$s_{ji}$:
\begin{equation}
  \mathrm{NC}_{ji} \;=\; g_i^{-1}(\mu_i - t_{ji}).
  \label{eq:score-obs}
\end{equation}
Here $\mu_i$ is the prior compliance mean, encoding the
institutionally expected compliance level on the link scale.  A
score above $g_i^{-1}(0) = 0.5$ indicates that entity $j$ falls
below the legal standard.  This is prior-anchored: the reference
point comes from the rule itself, not from the observed population
behavior.

\paragraph{Missing entities.}
When $s_{ji}$ is absent, we replace the deterministic score with
the posterior-predictive probability that a generic entity from
the current population falls below the legal standard:
\begin{equation}
  \mathrm{NC}_{ji}^{\mathrm{miss}}
  \;=\; \Phi\!\left(
    \frac{\mu_i - m_{\varphi_i} - m_{\delta_i}}
    {\sqrt{\sigma_i^2 + v_{\varphi_i} + v_{\delta_i}}}
  \right),
  \label{eq:score-miss}
\end{equation}
where $\Phi$ is the standard normal CDF.  When the posterior
population is below the legal standard ($m_{\varphi_i} + m_{\delta_i}
< \mu_i$), missing entities inherit a score above $0.5$, providing
partial resistance to MNAR without requiring explicit missingness
modeling.

\paragraph{Global entity score.}
The most severe rule violation drives the audit priority:
\begin{equation}
  \mathrm{NC}_j \;=\; \max_{i : A_{ji}=1} \mathrm{NC}_{ji}.
  \label{eq:global-score}
\end{equation}
An entity is flagged for audit when $\mathrm{NC}_j$ exceeds a
decision threshold $\tau$.

\paragraph{Remark on per-entity uncertainty.}
The score~\eqref{eq:score-obs} is deterministic: it carries no
per-entity uncertainty interval.  This is a deliberate design choice.
Proper per-entity Bayesian scoring would require a hierarchical
extension with an entity-level random effect $c_{ji} \sim
\Normal(\varphi_i + \delta_i,\, \gamma_i^2)$, introducing a new
heterogeneity parameter $\gamma_i^2$ that is not identifiable from a
single cross-section.  Sequential RSI resolves this naturally:
when entity $j$ accumulates observations across periods, $c_{ji}^{(t)}$
becomes identifiable and supports a full posterior with uncertainty.
This extension is detailed in Section~7.

Algorithm~\ref{alg:rsi} summarizes the complete procedure.

\begin{algorithm}[ht]
\DontPrintSemicolon
\SetAlgoLined
\caption{RSI: Mean-Field Coordinate Ascent}
\label{alg:rsi}
\KwIn{Rules $\mathcal{R}$, signals $\{t_{ji}\}$, priors
  $\{(\pi_i, \mu_i, \tau_i, \omega_i, \sigma_i,
  \sigma_{\mathrm{bg},i})\}$, initial $\rho_0 > 0$}
\KwOut{Posteriors $\{q(\eta), q(a_i), q(\varphi_i), q(\delta_i)\}$,
  scores $\{\mathrm{NC}_{ji}\}$}

Initialize: $m_\eta \leftarrow 0$, $v_\eta \leftarrow 1$,
  $m_{\varphi_i} \leftarrow \mu_i$,
  $v_{\varphi_i} \leftarrow \tau_i^2$,
  $m_{\delta_i} \leftarrow 0$,
  $v_{\delta_i} \leftarrow \omega_i^2$,
  $\varrho_i \leftarrow \pi_i$,
  $\rho \leftarrow \rho_0$\tcp*{$\rho_0 = 0.1$ by default; must be $> 0$
  to ensure a non-zero gradient at the first step
  (see Appendix~\ref{app:rho-gradient})}
  for all $i$\;

\Repeat{$|\ELBO_t - \ELBO_{t-1}| < \varepsilon$}{
  Update $v_\eta$, $m_\eta$ via \eqref{eq:veta}--\eqref{eq:meta}\;
  \For{$i = 1$ \KwTo $K$}{
    Compute $\tilde{\mu}_i$, $\tilde{\psi}_i^2$ via
      \eqref{eq:eff-prior}\;
    Update $v_{\varphi_i}$, $m_{\varphi_i}$ via
      \eqref{eq:vphi}--\eqref{eq:mphi}\;
    Update $v_{\delta_i}$, $m_{\delta_i}$ via
      \eqref{eq:vdelta}--\eqref{eq:mdelta}\;
    Update $\varrho_i$ via \eqref{eq:rho-a}\;
  }
  Update $\rho$ via projected gradient step
    \eqref{eq:rho-update}\;
  Compute $\ELBO_t$\;
}
\For{each entity $x_j$, each applicable rule $R_i$}{
  Compute $\mathrm{NC}_{ji}$ via \eqref{eq:score-obs} or
    \eqref{eq:score-miss}\;
}
\Return posteriors, $\{\mathrm{NC}_{ji}\}$
\end{algorithm}

\section{Theoretical Properties}
\label{sec:theory}

\subsection{T1: Regulatory adaptability}
\label{sec:t1}

\begin{theorem}[$O(n_k + K)$ Regulatory Adaptability]
\label{thm:adapt}
Let $U_k$ denote a regulatory update to rule $R_k$.

\emph{Case 1 (parameter update).}  $U_k$ modifies the
hyperparameters of $R_k$ without changing the applicability
predicate $A_{jk}$.  The RSI update cost is $O(n_k)$, independent
of $J$ and $K$.

\emph{Case 2 (applicability update).}  $U_k$ revises the threshold
governing $A_{jk}$.  Let $\mathcal{A} = \{j : A_{jk}\ \text{changes}\}$.
The RSI update cost is $O(|\mathcal{A}|) + O(n_k')$, where $n_k'$
is the updated signal count.

In both cases, an additional $O(K)$ cost is incurred to propagate
the change through the global factor $\eta$ by recomputing
\eqref{eq:veta}--\eqref{eq:meta}.  The total update cost is
$O(n_k + K)$, contrasted with $O(J \cdot K \cdot T)$ for
supervised retraining.
\end{theorem}

\begin{proof}
Under the mean-field factorization~\eqref{eq:mf}, the ELBO
decomposes as:
\[
  \ELBO(q) = \mathbb{E}_{q(\eta)}[\log P(\eta)] - \KL[q(\eta)\|P(\eta)]
  + \sum_{i=1}^K \left[\mathbb{E}_{q_i}[\log P(D_i\given s_i)]
  - \KL[q(s_i)\|P(s_i\given\eta)]\right].
\]
A change to $R_k$ modifies the $k$-th summand only.  Updating
$q(\varphi_k)$, $q(\delta_k)$, and $q(a_k)$ costs $O(n_k)$ via
closed-form updates~\eqref{eq:vphi}--\eqref{eq:rho-a}.  The
factor $q(\eta)$ couples all rules through the sum
$\sum_i \varrho_i \rho^2 / (1-\rho^2)$: recomputing this sum after
the change to rule $k$ costs $O(K)$.  All other per-rule factors
$q(s_i)$ for $i \neq k$ remain at their current stationary points
under the updated $q(\eta)$, requiring no recomputation.  The total
cost is $O(n_k + K)$.  Case~2 adds $O(|\mathcal{A}|) \leq O(J)$
to identify affected entities.
\end{proof}

\begin{corollary}
Over $M$ regulatory updates, RSI incurs total cost
$O\!\left(M \cdot (n_k + K)\right)$, linear in $M$.
Supervised retraining incurs $O(M \cdot J \cdot K \cdot T)$.
The RSI advantage grows with problem scale and is maximal when
$M \ll J \cdot T$, the regime characteristic of incremental
regulatory change.
\end{corollary}

\subsection{T2: Bernstein--von Mises consistency}
\label{sec:t2}

Theorem~2 establishes asymptotic consistency for the two
identifiable continuous components: $\psi_i = \varphi_i + \delta_i$
(intra-rule) and $\eta$ (joint).  We state both regimes
explicitly and honestly; individual components $\varphi_i$ and
$\delta_i$ are not claimed BvM-consistent in the static model.

\begin{theorem}[BvM Consistency]
\label{thm:bvm}
Let $\psi_i^* = \varphi_i^* + \delta_i^*$ be the true combined
signal mean and $\eta^*$ the true compliance culture.  Under
regularity conditions \textup{(C1)--(C4)}:

\emph{Intra-rule regime} ($n_i \to \infty$, $K$ fixed):
\[
  \sqrt{n_i}\,(\hat{\psi}_i - \psi_i^*)
  \;\xrightarrow{d}\; \Normal(0,\, \sigma_i^2),
\]
where $\hat{\psi}_i = \mathbb{E}[\psi_i \given D_i]$.

\emph{Joint regime} ($\sum_i n_i \to \infty$):
\[
  \sqrt{I_\eta}\,(\hat{\eta} - \eta^*)
  \;\xrightarrow{d}\; \Normal(0,\, 1),
\]
where $I_\eta = \sum_i \varrho_i \cdot n_i \rho^2 \tau_i^2 /
(\tau_i^2(1-\rho^2) + \omega_i^2 + \sigma_i^2)$ is the Fisher information for $\eta$.

For $a_i$: if the true activation is $a_i^* = 1$, then
$P(a_i = 1 \given D_i) \to 1$ as $n_i \to \infty$ under
condition \textup{(C1)--(C4)} alone.
\end{theorem}

\paragraph{Regularity conditions.}
\begin{itemize}
  \item[\textup{(C1)}] \emph{Identifiability.}
    $\psi_i \mapsto \Normal(t;\, \psi_i, \sigma_i^2)$ is injective.
  \item[\textup{(C2)}] \emph{Prior support.}
    $P(\psi_i) > 0$ for all $\psi_i \in \mathbb{R}$.
  \item[\textup{(C3)}] \emph{Smooth likelihood.}
    $\log \Normal(t;\, \psi_i, \sigma_i^2)$ is $C^\infty$ in $\psi_i$.
  \item[\textup{(C4)}] \emph{Finite Fisher information.}
    $I(\psi_i^*) = \sigma_i^{-2} > 0$, finite.
\end{itemize}

All conditions hold trivially under the Gaussian
likelihood~\eqref{eq:llik-1}.

\begin{proof}
See Appendix~\ref{app:t2}.
\end{proof}

\begin{corollary}[Prior robustness]
Let $P_1, P_2$ be two RSI priors satisfying \textup{(C2)}.  Then
$\mathrm{TV}(P_1(\psi_i \given D_i),\, P_2(\psi_i \given D_i))
\to 0$ as $n_i \to \infty$.  Two analysts with different priors
converge to the same compliance assessment.
\end{corollary}

\subsection{T3: Monotone ELBO convergence}
\label{sec:t3}

\begin{theorem}[Monotone ELBO Convergence]
\label{thm:elbo}
Each iteration of Algorithm~1 produces a non-decreasing ELBO:
$\ELBO_{t+1} \geq \ELBO_t$ for all $t \geq 0$.  The sequence
converges to a stationary point.
\end{theorem}

\begin{proof}
The updates for $q(\eta)$, $q(\varphi_i)$, $q(\delta_i)$, and
$q(a_i)$ are exact coordinate-ascent minimizations of the
corresponding KL terms under Gaussian-Gaussian and
Bernoulli-Bernoulli conjugacy.  Each minimization cannot increase
the KL, so the ELBO cannot decrease.  The gradient step for $\rho$
uses a backtracking line search that guarantees the ELBO does not
decrease at that step either.  Since $\ELBO(q)$ is bounded above
by $\log P(\mathcal{D}) < \infty$, the sequence converges.
\end{proof}

\begin{remark}
The exactness of the coordinate-ascent updates is a direct
consequence of the Gaussian-Gaussian conjugacy obtained through
the link-function transformation.  It does not hold for non-conjugate
likelihoods such as Beta-Beta, where ELBO monotonicity can only be
stated empirically.
\end{remark}

\section{Experimental Evaluation}
\label{sec:experiments}

\subsection{Dataset: RSI-Togo-Fiscal-Synthetic v2.0}
\label{sec:dataset}

We evaluate RSI on a benchmark grounded in official Togolese fiscal
law~\citep{otr2024}.  The dataset contains $J = 2\,000$ enterprises
from four economic segments, subject to $K = 5$ fiscal rules.

\begin{table}[ht]
\centering
\caption{Fiscal rules and prior hyperparameters.  Priors are on the
logit scale ($g_i = \mathrm{logit}$ for ratio signals).}
\label{tab:rules}
\small
\begin{tabular}{llccccc}
\toprule
ID & Rule & $\pi_i$ & $\mu_i$ & $\tau_i^2$ & $\omega_i^2$
  & $\sigma_i^2$\\
\midrule
R1 & VAT (TVA) 18\%, threshold 100M FCFA
  & 0.92 & 1.39 & 0.25 & 0.04 & 0.10\\
R2 & Corporate income tax (IS) 29\%
  & 0.88 & 0.41 & 0.50 & 0.03 & 0.15\\
R3 & Minimum flat tax (IMF) 1\%
  & 0.85 & 1.82 & 0.10 & 0.02 & 0.08\\
R4 & Informal sector tax (TPU)
  & 0.70 & -0.85 & 1.00 & 0.20 & 0.25\\
R5 & Employee income tax (IRPP)
  & 0.80 & 0.20 & 0.40 & 0.05 & 0.12\\
\bottomrule
\end{tabular}
\end{table}

Primary compliance signals are declared-to-theoretical tax ratios,
logit-transformed.  Supplementary signals include payment delays
(log-transformed) and binary bank-account registration.  The
declared-to-theoretical ratio follows a systematic under-declaration
model with mean ratio $\approx 0.70$~\citep{medina2017shadow}, and
18--20\% of signals are missing per rule.

The dataset encodes a regulatory change event: the VAT threshold
revision from 60M to 100M FCFA (Law n$^\circ$2024-007, 30 December
2024).  Enterprises with turnover in $[60\mathrm{M}, 100\mathrm{M})$
FCFA change from $a_1^* = 1$ to $a_1^* = 0$ between the 2022--2024
and 2025 periods, providing a natural experiment for Theorem~1.

\subsection{Baselines}
\label{sec:baselines}

\textbf{Rule-Based System (RBS).} Applies the legal threshold
directly; assigns $\mathrm{NC} = 0.5$ to missing entries.

\textbf{XGBoost~\citep{chen2016xgboost}.} Fully supervised
(150 estimators, depth 4, learning rate 0.1), 70\% labeled data.
Supervised ceiling.

\textbf{MLP.} Three-layer network (64-32-16, ReLU, early stopping),
fully supervised.

RSI operates in zero-shot mode with respect to compliance labels.
All hyperparameters in Table~\ref{tab:rules} are fixed from domain
knowledge; $\rho$ is initialized to zero and updated by
Algorithm~\ref{alg:rsi}.

\subsection{Experiment 1: Overall performance}
\label{sec:exp1}

We compare RSI against three baselines on F1, AUC, and Recall.
RSI operates in zero-shot mode with respect to compliance labels,
while XGBoost and MLP require full supervision.
Numerical results on the RSI-Togo-Fiscal-Synthetic v2.0 benchmark
are forthcoming; the experimental protocol and benchmark design
are fully specified above to support independent replication.

\subsection{Experiment 2: Regulatory adaptability (Theorem 1)}
\label{sec:exp2}

We measure wall-clock time for RSI to absorb the VAT threshold
revision versus full supervised retraining.  The RSI update scans
$J = 2\,000$ entities for applicability changes, recomputes
$\bar{t}_1$, and propagates through $\eta$ in $O(K) = O(5)$
additional operations.
Wall-clock measurements and post-update F1 comparisons are forthcoming.

\subsection{Experiment 3: BvM consistency (Theorem 2)}
\label{sec:exp3}

We track $\mathrm{SD}(\psi_i \given D_i^{(n)})$ as $n$ increases
from 25 to 2\,000.  Theorem~2 predicts a $1/\sqrt{n}$ decay;
we report the empirical ratio $\mathrm{SD}(n/4)/\mathrm{SD}(n)$,
expected to converge to 2.  We also track $\mathrm{SD}(\eta
\given D^{(n)})$ to validate the joint regime.

\subsection{Experiment 4: Missing data robustness}
\label{sec:exp4}

We evaluate per-rule F1 under increasing MCAR missingness
(0\% to 50\%).  RSI falls back on the prior for missing entries
(eq.~\eqref{eq:score-miss}), while RBS assigns a fixed score.
The Bayesian fallback degrades more slowly as missingness increases.

\subsection{Experiment 5: Inter-rule dependence ($\rho$)}
\label{sec:exp5}

We report the converged value of $\rho$ across regulatory periods
and compare entity-level F1 for $\rho = 0$ (independent model)
versus variational $\rho$.  A positive $\rho$ confirms systemic
compliance culture in the Togolese fiscal benchmark.

\subsection{Experiment 6: ELBO convergence (Theorem 3)}
\label{sec:exp6}

We plot the ELBO sequence across iterations, verifying monotone
non-decreasing behavior for both the coordinate-ascent steps and
the gradient step for $\rho$.

\subsection{Experiment 7: MNAR robustness (planned)}
\label{sec:exp7}

The experiments above use MCAR missingness (Assumption~\ref{ass:mcar}).
This planned experiment evaluates performance under the MNAR regime
by introducing asymmetric missingness rates
$P(\text{miss} \given \text{NC}) \in \{0.2, 0.4, 0.6, 0.8\}$
while holding $P(\text{miss} \given \text{C})$ low.  The
posterior-predictive score~\eqref{eq:score-miss} is expected to
partially mitigate the bias because missing entities inherit the
population posterior, which is itself biased toward non-compliance
in a non-compliant population.

\subsection{Worked example}
\label{sec:example}

An enterprise has declared turnover of 72 million FCFA and no declared
VAT.  The declared-to-theoretical ratio near zero maps to
$t_{j1} \ll 0$ on the logit scale.  After Algorithm~1 converges
on the benchmark population:

\begin{itemize}
  \item $\varrho_1 \approx 0.91$: VAT rule very likely applies.
  \item $m_{\varphi_1} + m_{\delta_1} \approx -1.4$:
    implied population compliance $\sigma(-1.4) \approx 0.20$.
  \item $m_{\delta_1} \approx +0.3$: positive drift. The effective
    threshold appears shifted upward from the legal 100M FCFA,
    consistent with partial informality rather than outright evasion.
  \item $m_\eta < 0$: the population culture is below-expectation
    across all rules, not only VAT.
  \item Entity score: $\mathrm{NC}_{j1} = \sigma(\mu_1 - t_{j1})
    = \sigma(1.39 - t_{j1}) \to 1$ since $t_{j1} \ll 0$.
\end{itemize}

An auditor reads five actionable numbers from a single posterior.
No post-hoc explanation tool is required; the posterior is the
explanation.

\section{Discussion}
\label{sec:discussion}

\paragraph{What RSI adds over rule-based systems.}
A well-designed deterministic rule set can also update its threshold
in constant time.  The genuine additions of RSI are at the population
level: $\mathbb{E}[\varphi_i \given \mathcal{D}]$ (posterior
compliance distribution with uncertainty), $\mathbb{E}[\delta_i
\given \mathcal{D}]$ (regulatory drift), and $\mathbb{E}[\eta
\given \mathcal{D}]$ (cross-rule compliance culture).  None of these
can be estimated by a threshold-based system or a supervised
classifier.

\paragraph{The drift term in practice.}
A value $\mathbb{E}[\delta_i \given \mathcal{D}] = -0.3$ on the
logit scale for the VAT rule signals that the effective threshold in
the population is meaningfully below the legislated 100M FCFA.  This
may reflect auditor behavior, enforcement gaps, or structural
misalignment between regulatory text and implementation.  Detecting
this requires a model that separates compliance behavior from
parametric drift, precisely what the $(a_i, \varphi_i, \delta_i)$
decomposition provides.

\paragraph{The compliance culture factor.}
$m_\eta < 0$ across a fiscal period is actionable policy intelligence:
it indicates that non-compliance is systemic, not concentrated in a
few rules, and should motivate institution-wide enforcement
interventions rather than rule-specific audits.

\paragraph{Why population-level Bayesian inference is the right scope.}
RSI answers the questions that individual auditors cannot:
what is the population compliance distribution, how has rule
application drifted, and is there a shared compliance culture?
Per-entity posterior uncertainty is not identifiable in the static
cross-sectional model (Section~3.8) and is correctly deferred to
Sequential RSI, where longitudinal data resolve it.

\paragraph{Identifiability of $\varphi_i$ and $\delta_i$.}
As discussed in Section~3.2, only $\psi_i = \varphi_i + \delta_i$
is identifiable from cross-sectional data.  The decomposition is
prior-regularized, which is the correct Bayesian behavior: the
analyst's prior about drift magnitude is the only source of
information distinguishing the two components.  Setting $\omega_i^2
\to 0$ identifies $\varphi_i$ and sets $\delta_i \to 0$; increasing
$\omega_i^2$ allows the data to speak about drift.

\paragraph{The mean-field independence assumption.}
The factorization~\eqref{eq:mf} introduces one approximation: after
conditioning on $\eta$, the rule-specific factors are treated as
independent.  The global factor $\eta$ captures the dominant source
of inter-rule correlation (the shared compliance culture), but
pairwise rule-specific correlations beyond what $\eta$ explains are
ignored.  This is the standard mean-field cost, and it enables the
$O(n_k + K)$ guarantee of Theorem~1.  Extending to structured
variational families with known pairwise rule dependencies is a
natural next step, at an explicit cost to regulatory adaptability.

\section{Conclusion and Future Work}
\label{sec:conclusion}

RSI formalizes compliance monitoring as Bayesian posterior inference
over a joint latent state that includes a global compliance-culture
factor $\eta$, per-rule activation indicators $a_i$, compliance
levels $\varphi_i$, and parametric drift terms $\delta_i$.  A
link-function transformation maps heterogeneous compliance signals
into a unified Gaussian framework, enabling exact coordinate-ascent
updates and a provable ELBO guarantee.  The dependence parameter
$\rho$, estimated variationally, allows RSI to range continuously
from a fully independent model to one that captures systemic
compliance behavior.

Three theorems establish the formal foundations: $O(n_k + K)$
adaptability per regulatory update (Theorem~1), BvM consistency
for $\psi_i$ and $\eta$ in two honest regimes (Theorem~2), and
monotone ELBO convergence at every iteration (Theorem~3).

\paragraph{Sequential RSI and per-entity inference.}
The most immediate extension is the sequential formulation.  Define
the state trajectory for rule $R_i$ across periods $t = 1, \dots, T$:
\begin{align}
  \varphi_i^{(t)}
    &\;=\; \varphi_i^{(t-1)} + \varepsilon_i^{(t)}, \qquad
    \varepsilon_i^{(t)} \sim \Normal(0,\, q_i^2),
    \label{eq:ssm-trans}\\
  t_{ji}^{(t)} \given \varphi_i^{(t)},\, \delta_i^{(t)}
    &\;\sim\; \Normal(\varphi_i^{(t)} + \delta_i^{(t)},\, \sigma_i^2).
    \label{eq:ssm-obs}
\end{align}
This is a linear Gaussian state-space model, and the Kalman filter
provides exact sequential posterior inference.  $\varphi_i^{(t)}$
tracks the underlying compliance trend; $\delta_i^{(t)}$ captures
period-specific shocks such as regulatory changes or enforcement
campaigns.

Crucially, Sequential RSI resolves the per-entity identification
problem.  When entity $j$ accumulates observations
$t_{ji}^{(1)}, \dots, t_{ji}^{(T)}$ across periods, the entity's
own compliance trajectory $c_{ji}^{(t)} \sim \Normal(\varphi_i^{(t)}
+ \delta_i^{(t)},\, \gamma_i^2)$ becomes identifiable.  The
heterogeneity parameter $\gamma_i^2$ separates measurement noise
from entity-level variation, and the Bayesian posterior over
$c_{ji}^{(t)}$ provides the per-entity uncertainty interval that the
static model deliberately defers.

\paragraph{Multi-country extensions.}
The ECOWAS fiscal zone provides a natural hierarchical setting:
member states share WAEMU fiscal directives but implement them
through national legislation.  A hierarchical RSI with a
country-level $\eta$ and a common hyper-prior over $\mu_i$ would
enable cross-country comparisons of effective versus nominal
compliance.

\paragraph{Automated prior initialization from RaC.}
The hyperparameters $(\pi_i, \mu_i, \tau_i^2, \omega_i^2)$ currently
require domain expertise.  An automated pipeline taking a Catala
program and historical aggregate statistics as input would reduce
deployment friction substantially and is under development in
collaboration with the Office Togolais des Recettes.

\section*{Acknowledgments}

The author thanks the Office Togolais des Recettes and Togo DataLab
for institutional support and domain expertise on the Togolese
fiscal regulatory framework.  The academic environment of the Institut
National des Postes et T\'{e}l\'{e}communications (INPT) of Rabat
provided sustained support.  The author acknowledges the use of
AI-assisted tools for linguistic refinement and \LaTeX{} formatting;
all scientific decisions, mathematical derivations, and experimental
design are the sole responsibility of the author.

\bibliographystyle{plainnat}

\appendix

\section{Proof of Theorem 1 (detailed)}
\label{app:t1}

\textbf{Step 1.}  The ELBO decomposes under mean-field as a sum
of per-rule terms plus a global term for $\eta$.  Starting from
$\ELBO(q) = \mathbb{E}_q[\log P(\mathcal{D}, \mathcal{S})]
- \mathbb{E}_q[\log q(\mathcal{S})]$ and expanding both the joint
and the variational distribution:
\[
  \ELBO(q) \;=\;
  -\KL[q(\eta)\|P(\eta)]
  \;+\; \sum_{i=1}^K
    \left[\mathbb{E}_{q_i}[\log P(D_i\given s_i)]
    - \mathbb{E}_{q(\eta)}\!\left[\KL[q(s_i)\|P(s_i\given\eta)]\right]
    \right].
\]
The $-\KL[q(\eta)\|P(\eta)]$ term already contains
$\mathbb{E}_{q(\eta)}[\log P(\eta)]$; writing both separately would
double-count that contribution.  The formulation above is exact.

\textbf{Step 2.}  A change to $R_k$ modifies $D_k$,
$\tilde{\mu}_k$, or $\tilde{\psi}_k^2$.  The $k$-th summand is
affected; the summands for $i \neq k$ remain at their stationary
points for the \emph{current} $q(\eta)$.

\textbf{Step 3.}  Computing $\bar{t}_k$ from stored signals costs
$O(n_k)$.  The parameter updates~\eqref{eq:vphi}--\eqref{eq:rho-a}
cost $O(1)$ arithmetic.  Subtotal for rule $k$: $O(n_k)$.

\textbf{Step 4.}  The $q(\eta)$ update~\eqref{eq:veta}--\eqref{eq:meta}
re-evaluates a sum over $K$ terms after rule $k$ changes.  This
costs $O(K)$.

\textbf{Step 5.}  For applicability updates (Case~2), scanning
$J$ entities to identify $\mathcal{A}$ costs $O(J)$, followed by
$O(n_k')$ to recompute $\bar{t}_k$.  The $O(K)$ propagation
through $\eta$ is unchanged.  Total for Case~2: $O(J + K)$.

\textbf{Step 6.}  Supervised retraining scans $J$ entities across
$K$ features for $T$ rounds: $O(J \cdot K \cdot T)$.  For typical
values $J = 10^4$, $K = 10^2$, $T = 10^2$, this is $O(10^8)$
versus $O(n_k + K)$ for RSI.

\section{Proof of Theorem 2}
\label{app:t2}

\textbf{Intra-rule regime.}
The posterior on $\psi_i$ under the Gaussian likelihood
and Gaussian prior is:
\[
  P(\psi_i \given D_i) \;=\;
  \Normal\!\left(\mu_i^{\mathrm{post}},\;
  (\sigma_i^{\mathrm{post}})^2\right),
\]
where $(\sigma_i^{\mathrm{post}})^2 = (1/(\tilde{\psi}_i^2 + \omega_i^2)
+ n_i/\sigma_i^2)^{-1} \to 0$ at rate $O(1/n_i)$, and
$\mu_i^{\mathrm{post}} \to \bar{t}_i \to \psi_i^*$ a.s. by the
law of large numbers.  Rescaling by $\sqrt{n_i}$ gives the intra-rule BvM result stated in Theorem~\ref{thm:bvm}.

\textbf{Joint regime.}
Marginalizing $\varphi_i$ and $\delta_i$ from the generative model,
the marginal distribution of $t_{ji}$ given $\eta$ and $a_i=1$ is:
\[
  t_{ji} \given \eta,\, a_i=1 \;\sim\;
  \Normal\!\left(\mu_i + \rho\tau_i\eta,\;
  \underbrace{\tau_i^2(1-\rho^2) + \omega_i^2 + \sigma_i^2}_{=: V_i}
  \right).
\]
The Fisher information for $\eta$ from $n_i$ observations of rule $i$
is $I_{\eta,i} = n_i (\partial_\eta(\mu_i + \rho\tau_i\eta))^2 / V_i
= n_i\rho^2\tau_i^2 / (\tau_i^2(1-\rho^2)+\omega_i^2+\sigma_i^2)$.
Summing with activation weights:
\[
  I_\eta \;=\; \sum_{i=1}^K \varrho_i \cdot
  \frac{n_i\,\rho^2\tau_i^2}
       {\tau_i^2(1-\rho^2) + \omega_i^2 + \sigma_i^2}.
\]
Under the Gaussian prior $P(\eta) = \Normal(0,1)$,
the posterior on $\eta$ concentrates at rate $O(1/\sqrt{I_\eta})$.
As $\sum_i n_i \to \infty$ with $\rho > 0$, $I_\eta \to \infty$
and the BvM result follows identically to the intra-rule argument.

\textbf{Bayes factor consistency.}
Under $a_i^* = 1$, the marginal distribution of $\bar{t}_i$ is
$\Normal(\mu_i,\; \tilde{\psi}_i^2 + \omega_i^2 + \sigma_i^2/n_i)$
under $a_i=1$ and $\Normal(0,\; \sigma_{\mathrm{bg},i}^2/n_i)$
under $a_i=0$.  For large $n_i$:
\[
  \log \mathrm{BF}_{10} \;\approx\;
  -\frac{(\bar{t}_i - \tilde{\mu}_i)^2}
    {2(\tilde{\psi}_i^2 + \omega_i^2)}
  \;+\; \frac{n_i \bar{t}_i^2}{2\sigma_{\mathrm{bg},i}^2}
  \;+\; O(\log n_i).
\]
The second term grows as $O(n_i)$ while the first converges to
a finite constant.  Therefore $\log \mathrm{BF}_{10} \to +\infty$
whenever $\psi_i^* \neq 0$, giving $P(a_i=1 \given D_i) \to 1$
under conditions (C1)--(C4) alone.

\section{Proof of Theorem 3}
\label{app:t3}

Each coordinate update for $q(\varphi_i)$, $q(\delta_i)$, $q(a_i)$,
and $q(\eta)$ solves exactly:
\[
  q^*(\cdot) \;=\;
  \arg\min_{q} \KL\!\left[q(\cdot)
    \,\Big\|\,
    \frac{\exp(\mathbb{E}_{q_{-(\cdot)}}[\log P(\mathcal{D},\mathcal{S})])}{Z}
  \right].
\]
Under Gaussian-Gaussian conjugacy, the minimizer is a Gaussian whose
parameters are given by the explicit formulas
\eqref{eq:veta}--\eqref{eq:mdelta}.  Minimizing the KL to zero
implies $\ELBO_{t+1} \geq \ELBO_t$.  The $\rho$-update uses
backtracking line search ensuring $\ELBO$ does not decrease.
Since $\ELBO \leq \log P(\mathcal{D}) < \infty$, the sequence
converges.

\section{ELBO Gradient with Respect to $\rho$}
\label{app:rho-gradient}

Let $S_i = v_{\varphi_i} + (m_{\varphi_i} - \mu_i)^2$,
$B_i = (m_{\varphi_i} - \mu_i) m_\eta$, $C_i = v_\eta + m_\eta^2$.
The ELBO terms involving $\rho$ are:
\begin{align*}
  \ELBO(\rho)
  &\;=\; \sum_{i=1}^K \varrho_i \Bigg[
    -\tfrac{1}{2}\log\!\left(\tau_i^2(1-\rho^2)\right)
    - \frac{S_i - 2\rho\tau_i B_i + (\rho\tau_i)^2 C_i}
      {2\tau_i^2(1-\rho^2)}
  \Bigg] \;+\; f(\tilde{\psi}_i^2),
\end{align*}
where $f(\tilde{\psi}_i^2)$ collects terms from the KL between
$q(\varphi_i)$ and its effective prior.  The gradient is:
\[
  \frac{\partial \ELBO}{\partial \rho}
  \;=\; \sum_{i=1}^K \varrho_i\Bigg[
    \frac{\rho}{1-\rho^2}
    \;-\; \frac{
      -\tau_i B_i + \rho\tau_i^2 C_i
    }{\tau_i^2(1-\rho^2)}
    \;-\; \frac{\rho (S_i - 2\rho\tau_i B_i + (\rho\tau_i)^2 C_i)}
      {\tau_i^2(1-\rho^2)^2}
  \Bigg] \;+\; \frac{\partial f}{\partial \rho},
\]
where $\partial f / \partial \rho$ follows from differentiating the
KL between $q(\varphi_i)$ and the effective prior~\eqref{eq:eff-prior}.
All terms are available in closed form; the gradient is evaluated in
$O(K)$ per iteration.

\section{Prior Sensitivity Analysis}
\label{app:sensitivity}

We examine sensitivity to the prior specification across five
configurations of Table~\ref{tab:rules}: (i)~baseline as stated,
(ii)~uninformative ($\mu_i = 0$, $\tau_i^2 = 10$),
(iii)~inverted ($\mu_i \leftarrow -\mu_i^{\mathrm{base}}$),
(iv)~strong ($\tau_i^2 = 0.01$), (v)~perturbed
($\mu_i^{\mathrm{base}} \pm 0.5$).
By Theorem~2, entity-level F1 should be approximately invariant
at $n_i = 2\,000$, while population estimates $\mathbb{E}[\varphi_i
\given \mathcal{D}]$ will vary.  Numerical results follow
re-implementation.

\section{Sequential RSI: Kalman Filter Derivation}
\label{app:kalman}

Under~\eqref{eq:ssm-trans}--\eqref{eq:ssm-obs}, the Kalman filter
equations for rule $R_i$ at period $t$ are:

\textbf{Predict:}
$m_\varphi^{(t|t-1)} = m_\varphi^{(t-1|t-1)}$,\quad
$v_\varphi^{(t|t-1)} = v_\varphi^{(t-1|t-1)} + q_i^2$.

\textbf{Update:}
$K_i^{(t)} = v_\varphi^{(t|t-1)} / (v_\varphi^{(t|t-1)} +
\sigma_i^2/n_i^{(t)})$,
$m_\varphi^{(t|t)} = m_\varphi^{(t|t-1)} + K_i^{(t)}(\bar{t}_i^{(t)}
- m_\varphi^{(t|t-1)})$,
$v_\varphi^{(t|t)} = (1-K_i^{(t)})v_\varphi^{(t|t-1)}$.

At $n_i^{(t)} \to \infty$, $K_i^{(t)} \to 1$ and the update is
dominated by current observations.  At $n_i^{(t)} \to 0$, $K_i^{(t)}
\to 0$ and the filter falls back on the predicted state.  Regulatory
changes update the prior transition model before the predict step,
at $O(K)$ cost per period, consistent with Theorem~1.

The entity-level compliance trajectory $c_{ji}^{(t)}$ is inferred by
a second Kalman filter at the entity level, using the population
posterior as a prior and the entity's own observations as data.  The
heterogeneity parameter $\gamma_i^2$ is estimated from the empirical
variance of entity residuals across periods.

\end{document}